\documentclass{article}

\usepackage[preprint]{neurips_2022}

\usepackage{amssymb}
\usepackage{amsmath}
\DeclareMathOperator*{\argmax}{arg\,max}
\DeclareMathOperator*{\argmin}{arg\,min}
\usepackage{MnSymbol}
\usepackage{natbib}
\usepackage{xcolor}
\usepackage{multirow}
\usepackage[normalem]{ulem}
\useunder{\uline}{\ul}{}
\usepackage{amsmath}
\usepackage{amssymb}
\usepackage{amsthm}
\usepackage{algorithm}
\newtheorem{theorem}{Theorem}[section]

\newtheorem{lemma}[theorem]{Lemma}
\newtheorem{remark}[theorem]{Remark}
\usepackage[noend]{algpseudocode}
\usepackage{bm}
\usepackage{graphicx}

\usepackage[utf8]{inputenc} 
\usepackage[T1]{fontenc}    
\usepackage{hyperref}       
\usepackage{url}            
\usepackage{booktabs}       
\usepackage{amsfonts}       
\usepackage{nicefrac}       
\usepackage{microtype}      
\usepackage{xcolor}         

\title{Decomposable Sparse Tensor on Tensor Regression}

\author{%
 Haiyi Mao \\
  Department of System and Computational Biology\\
  School of Medicine\\
  University of Pittsburgh\\
  \texttt{ham112@pitt.edu} \\
  \And
  Jason Xiaotian Dou \\
  Department of Electrical and Computer Engineering, \\
  University of Pittsburgh\\
  \texttt{jason.dou@pitt.edu} 
}

\begin{document}

\maketitle

\begin{abstract}
  Most regularized tensor regression research focuses on tensors predictors with scalars responses or vectors predictors to tensors responses. We consider the sparse low rank tensor on tensor regression where predictors $\mathcal{X}$ and responses $\mathcal{Y}$ are both high-dimensional tensors. By demonstrating that the general inner product  \citep{raskutti2019convex} or the contracted product \citep{lock2018tensor} on a unit rank tensor can be decomposed into standard inner products and outer products, the problem can be simply transformed into a tensor to scalar regression followed by a tensor decomposition. So we propose a fast solution based on stagewise search composed by contraction part and generation part which are optimized alternatively. We successfully demonstrate that our method can outperform current methods in terms of accuracy, predictors selection by effectively incorporating the structural information.
\end{abstract}

\section{Introduction}
Regularized regression plays a fundamental role in statistics, machine learning, data mining to identify the relationship between the predictors $X$ and responses $Y$. \citep{hastie} demonstrated the sparse interpretation of the coefficient between $X$ and $Y$ is one of the most important criteria for regression models.  LASSO \citep{tibshirani1996regression} by minimizing the sum of squares of residuals subject to $\ell_1$ regularizer, has the sparse estimation of the coefficient thus implicitly feature selection.  \par
Tensor data are also called as multi-dimensional or multi-way arrays to represent higher dimensional complex structural data which increasingly become common nowadays \citep{carroll1970analysis, kolda2009tensor, zhou2013tensor}. For instance, it is common that people collect high dimensional omic data over multiple perspective like times points, tissues, fluids \citep{ramasamy2014genetic}. In medical imaging, Magnetic Resonanace Imaging (MRI) scans, computerized tomography (CT) scans, are usually represented as tensors with dimensions that represent regions, subjects, and tissues \citep{de2011predicting}. Moreover, in deep learning architecture, the feature or the internal layers are represented as tensors as well \citep{lecun2015deep}. It becomes increasingly important to have regression models with tensor predictors and tensor responses. For example regressions between fMRI data and EEG data can capture the relation between spatial and temporal information \citep{de2011predicting, jansen2012motion, huster2012methods}. Another example is the regressions from gene expression data from multiple tissues or cells to another genetic variables could indicate the polymorphisms among gene expressions \citep{lock2018tensor, gtex2015genotype}. \par. 
The regularized tensor regressions have been intensively studied. While most of them are either from tensor to scalar \citep{yu2016learning, he2018boosted, zhou2013tensor}, or from scalar to tensor \citep{sun2017provable, sun2017store, li2017parsimonious}. \citep{lock2018tensor, raskutti2019convex} define the tensor on tensor regression. However  \citep{lock2018tensor} only limits to $\ell_2$ regularizer, and \citep{raskutti2019convex} gives the risk bound for $\ell_1$ norm regularizer estimation\par
In this paper, we introduce the decomposable sparse tensor on tensor regression. This is one of the first methods to solve the sparse tensor on tensor regression subject to the low rankness constraint. By rigorously proving that learning a unit rank coefficient tensor can be decomposed into optimizing a standard tensor to scalar regression followed by a tensor decomposition, we intuitively reduce the complexity of the problem. Furthermore, inspired by \citep{hastie2007forward, he2018boosted}, the stagewise search method are adopted to find the local minimum. Experiments under different settings show our method out perform previous methods by incorporating structural information from both predictors and responses.  
\section {Related Work}
Recently the tensor regression and its related applications have been intensively studied. 
\citep{zhou2013tensor, yu2016learning, he2018boosted} propose the low rank tensor regression framework where predictor $\mathcal{X}$ is a tensor and  scalar response $y$ as (\ref{eq:rg1})
\begin{equation} \label{eq:rg1}
\begin{aligned}
    &\mathcal{B} = \argmin_{\mathcal{B}}\mathcal{L}(\mathcal{B};\mathcal{X}, y) + \lambda \Omega(\mathcal{B})\\
    & s.t. rank(\mathcal{B}) \leq R
  \end{aligned}
\end{equation}
where the $\mathcal{X}$ is a tensor with size $M \times d_{1} \times \hdots \times d_{n}$; $y$ is a scalar response vector with size $M \times 1$; $\mathcal{B}$ is the coefficient tensor with size $d_{1} \times \hdots \times d_{n}$ subject to the low rankness constraint;  $\mathcal{L}(\cdot)$ is the loss function, and $\Omega(\cdot)$ is the regularizers which usually are $\ell_{1}$ or $\ell_{2}$ norms.\par
\citep{kolda2009tensor, rabusseau2016low, sun2017store, li2017parsimonious} propose the tensor decomposition models under low rankness constraint(CP decomposition or Tucker decomposition)as following (\ref{1})
\begin{equation} \label{1}
\begin{aligned}
    &\mathcal{B} = \argmin_{\mathcal{B}}\mathcal{L}(\mathcal{B}\times_{n+1} X, \mathcal{Y}) + \lambda \Omega(\mathcal{B})\\
    & s.t. rank(\mathcal{B}) \leq R
  \end{aligned}
\end{equation}

where the $X$ is a matrix with size $d_{p}\times M$; $\mathcal{Y}$ is a tensor responses vector with size $d_{1} \times \hdots \times d_{n}\times M$; $\mathcal{B}$ is the regression tensor with size $d_{1}\times d_{2} \times \hdots \times d_{n}\times d_{p}$ subject to the low rankness constraint; $\mathcal{L}$ is the loss function; and $\Omega(\cdot)$ is the $\ell_{1}$ or $\ell_{2}$ norm regularizer. 

Different with previous works, \citep{lock2018tensor,raskutti2019convex} extend the tensor on scalar or the tensor decomposition (the vectors on tensor regression) to tensor on tensor regression under different constraints. This regression is defined based on contracted tensor product $\llangle \cdot \rrangle_{p}$. The problem is formulated as following 
\begin{equation}
\begin{aligned}
    &\mathcal{B} = \argmin_{\mathcal{B}}\mathcal{L}(\llangle \mathcal{B},  \mathcal{X} \rrangle_{P}, \mathcal{Y}) + \lambda \Omega(\mathcal{B})\\
    & s.t. rank(\mathcal{B}) \leq R
  \end{aligned}
\end{equation}
Notice that the tensor on tensor regression in \citep{lock2018tensor} is only constrained on the $\ell_{2}$ norm. The solution is based on gibbs sampling which is slow and unstable. In this paper, we simplify the problem into two sub-problems under both $\ell_{2}$ and $\ell_{1}$ constraints.  
\section{Preliminaries}
Multi-dimensional array $\mathcal{A} \in \mathbb{R}^{p_{1}\times p_{2}\times ... \times p_{d}}$ is a $p_{1} \times p_{2} \times ... \times p_{d}$ dimension tensor. We introduce some basic tensor operations which are essential for our model formultiona.
\paragraph{Tensor Outer Product} First we define vector outer product. Given two vectors, $a$ with size $n \times 1$ and $b$ with size $m \times 1$. So the outer product of $a$ and $b$ is a $n \times m$ matrix is defined as following
\begin{equation}
    a\otimes b = a \circ b=
    \begin{bmatrix}
     a_{1}b_{1} & a_{1}b_{2}& ...& a_{1}b_{m} \\ 
     a_{2}b_{1} & a_{2}b_{2}& ...& a_{2}b_{m} \\ 
    \vdots&\vdots &\vdots&\vdots \\
     a_{n}b_{1} & a_{n}b_{2}& ...& a_{n}b_{m} \\ 
     \end{bmatrix}
\end{equation}
Then we expand two vectors outer product to k vectors $k > 2$  outer product. $\mathcal{A} = a_{1} \otimes a_{2} \otimes \hdots \otimes a_{k}$  or  $\mathcal{A} = a_{1} \circ a_{2} \circ \hdots \circ a_{k}$ where $a_{i}$, $1 \le i \le k$ is with size $d_{i} \times 1$. So $\mathcal{A}$ is a tensor with size $d_{1} \times d_{2} \hdots \times d_{k}$. Generally, given two tensors $\mathcal{A}$ with size $d_{1} \times \hdots d_{p}$ and $\mathcal{B}$ with size $d_{p+1} \times \hdots d_{p+q}$, $\mathcal{A} \otimes \mathcal{B}$ is a new tensor $\mathcal{C}$ with size $d_{1} \times \hdots d_{p} \times d_{p+1} \times \hdots d_{p+q}$. To be specific, the element of $\mathcal{C}$ can be computed as following 
\begin{equation}
    \mathcal{C}[p_{1}, \hdots, p_{n}, q_{1}, \hdots, q_{m}] = \mathcal{A}[p_{1}, \hdots, p_{n}] \mathcal{B}[q_{1}, \hdots, q_{m}]
\end{equation}
\paragraph{Tensor Inner Product} $\langle \mathcal{X}, \mathcal{Y}\rangle$ if $\mathcal{X}$ and $\mathcal{Y}$ have the same dimension. The tensor inner product can be transformed into vector inner product. 
\begin{equation}
    \langle \mathcal{X}, \mathcal{Y}\rangle = \langle \bf{Vec}(\mathcal{X}), \bf{Vec}(\mathcal{Y})\rangle
\end{equation}
where Vec($\mathcal{X}$) is an operation which flattens a $d_{1}\times d_{2}\times ... \times d_{n}$ tensor to a $d_{1}d_{2}...d_{p} \times 1$ vector. 
\paragraph{Contracted Tensor Product} named by \citep{lock2018tensor} or the general tensor inner product \citep{raskutti2019convex} is that two tensors $\mathcal{A} \in \mathbb{R}^{d_{1}\times ... \times d_{p}\times d_{p+1} \times ... \times d_{p+q} }$ and  $\mathcal{B} \in \mathbb{R}^{d_{p+1}\times ...\times d_{p+q}\times d_{p+q+1}\hdots \times d_{p+q+r}}$,
$\llangle \mathcal{A}, \mathcal{B} \rrangle_{Q} \in \mathbb{R}^{ d_{1} \times ...\times d_{p}\times d_{p+q+1}\times  ... \times d_{p+q+r}} $, where $Q$  denotes the first q modes product. It can be seen as a general matrix product as (\ref{cprod}). 
\begin{equation}
\label{cprod}
\begin{aligned}
    &\llangle \mathcal{A}, \mathcal{B} \rrangle_{Q} [i_{1},\hdots,i_{p}, i_{p+q+1},\hdots,i_{p+q+r}] \\
    &= \sum_{j_{1}=1}^{d_{p+1}}
    ...\sum_{j_{q}=1}^{d_{p+q}} \mathcal{A}[i_{1},\hdots,i_{p}, j_{1}...j_{q}]\mathcal{B}[j_{1}, \hdots, j_{q}, i_{p+q+1} \hdots i_{p+q+r}]
    \end{aligned}
\end{equation}


\paragraph{CP Decomposition} approximates a tensor with a summation of rank-one vectors outer production. The rank of the decomposition is simply the number of the rank-one tensors used to approximate the input tensor: given an input tensor $\mathcal{X} \in \mathbb{R}^{p_{1}\times p_{2}\times ... \times p_{d}}$
\begin{equation}\label{eq:2}
    \mathcal{X} = \sum_{r=1}^{R} \mathbf{a_{1}^{r} \circ a_{2}^{r}\circ \hdots \circ a_{d}^{r}}
\end{equation}
In (\ref{eq:2}) the outer product $\mathbf{a_{1}^{r} \circ a_{2}^{r}\circ \hdots \circ a_{d}^{r}}$ is a $p_{1} \times p_{2} \times \hdots \times p_{d}$ array with entries $(\mathbf {a_{1} \circ a_{2}\circ \hdots \circ a_{d}})_{(i_1\hdots i_d)} = \prod_{k=1}^{d}a_{i_{k}}$ where $\mathbf{a_{k}^{r}} \in \mathbb{R}^{p_{j}}$, $k = 1,2,\hdots , d$. We denote the tensor $\mathcal{X}$ has rank R. The $\mathcal{X} = [\![\mathbf{A}_{1}, \mathbf{A}_{2}, \hdots \mathbf{A}_{d} ]\!]$ where $\mathbf{A}_{k} = [\mathbf{a}_k^{(1)}, \mathbf{a}_k^{(2)}, \hdots, \mathbf{a}_k^{(R)}] \in \mathbb{R}^{p_{k} \times R}$, $k = 1,2,\hdots d$. A lot previous tensor to scalar regression models are based on CP decompostion \citep{he2018boosted, li2017parsimonious, peng2010regularized} 

\paragraph{Tensor Matricization} is transforming a tensor to a matrix which is the operation of reordering the elements of an N-way array into a matrix. Specifically, for a mode-n matricization of $\mathcal{A} \in \mathbb{R}^{d_{1}\times \hdots \times d_{n}}$ is denoted by $A_{(n)}$ with dimension $d_{n} \times d_{1}d_{2}\hdots d_{n-1}$. The elements of $\mathcal{A}[i_{1}, \hdots i_{n}]$ maps to $A[i_{n}, j]$, where
\begin{equation*}
j = 1 + \sum\limits_{\substack{l \neq n \\ l=1}}^{n} (i_{k}-1)J_{k}, \, J_{k} = \prod\limits_{\substack{m=1 \\ m\neq n}}^{k-1}d_{m}
\end{equation*}
\begin{figure} \label{mdl1}
\begin{center}
\includegraphics[scale=0.4]{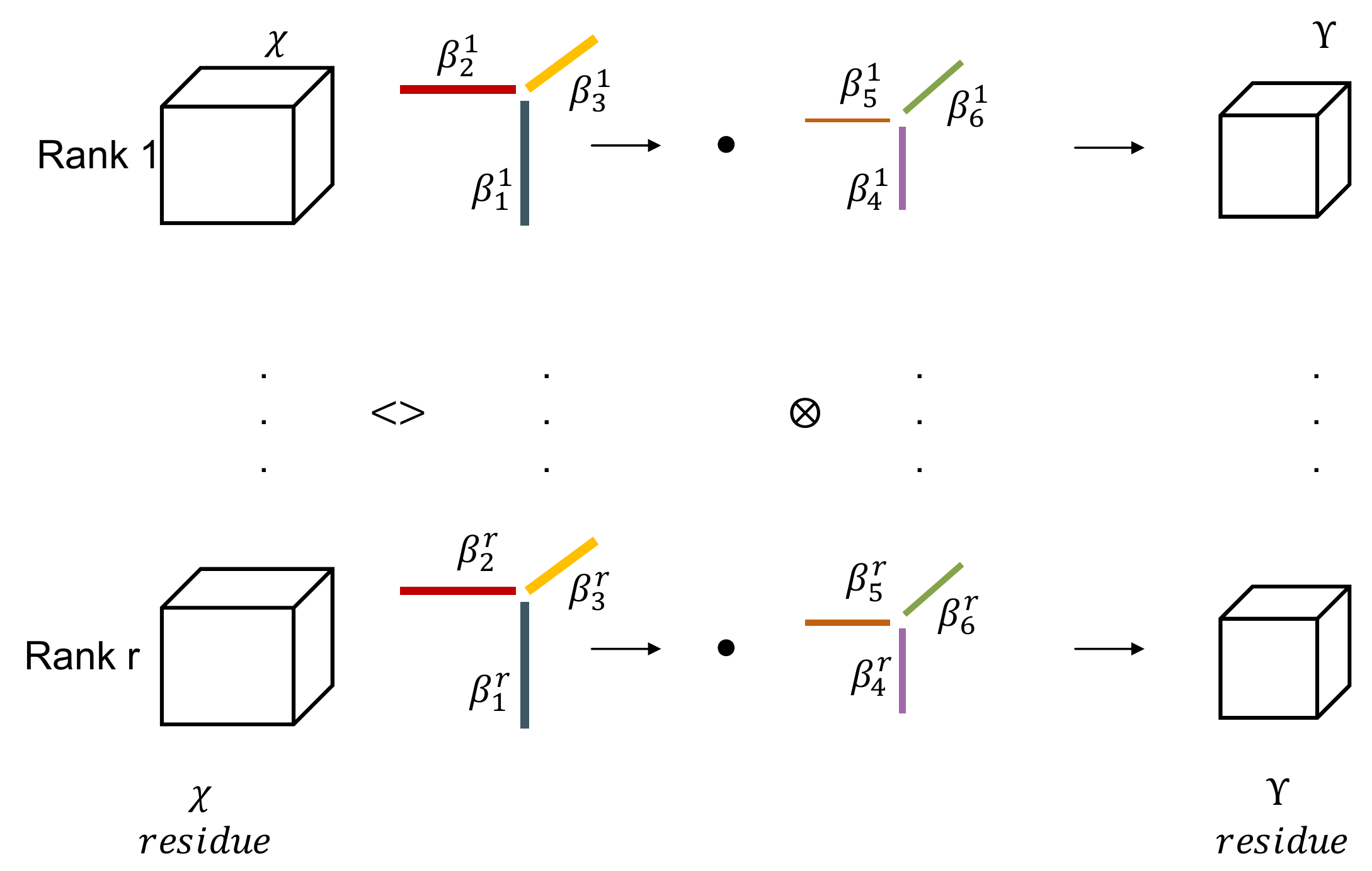}
\caption{The illustration of decomposable 3 modes (3D) tensor on 3D tensor regression. For each rank,  input tensors (residue) $\mathcal{X}$ do inner product with $\mathcal{B}_{P}^{r} = \pmb{\beta}^{r}_{1} \circ \pmb{\beta}^{r}_{2} \circ \pmb{\beta}^{r}_{3}$ to a scalar (black dots). Then the scalar do outer product with $\mathcal{B}_{Q}^{r} = \pmb{\beta}^{r}_{4} \circ \pmb{\beta}^{r}_{5} \circ \pmb{\beta}^{r}_{6}$ }
\end{center}
\end{figure}
\section{Decomposable Sparse Tensor on Tensor Regression}
\subsection{Model Formulation}
For a $m$th data point  $\mathcal{X}_{m}$ $\in \mathbb{R}^{d_{1}\times d_{2} \times \hdots \times d_{p}}$ and a tensor response $m$th response $ \mathcal{Y}_{m}\in \mathbb{R}^{d_{p+1}\times d_{p+2} \times  \hdots \times d_{p+q}}$.  $\mathcal{F}$ is a linear transformation function in tensor space. 
\begin{equation}\label{eq:pr}
  \mathcal{Y}_{m} = \mathcal{F}(\mathcal{X}_{m}) + \varepsilon^{m}
\end{equation}

Where $\varepsilon^{m} \in \mathbb{R}^{d_{p+1}\times d_{p+2} \times  \hdots \times d_{p+q}}$ is the tensor of noise term whose entries are independent and identically distributed centered Gaussian distribution with zero mean and variance $\sigma^{2}$ independent with $\mathcal{X}_{m}$.
Without loss of generality, the intercept is set to zero by centering the response and standardizing the predictors $\sum_{m=1}^{M}{\mathcal{Y}_{m}} = \mathbf{0}$, $\sum_{m=1}^{M}{\mathcal{X}_{m}} = \mathbf{0}$. Then the function $\mathcal{F}$ can be formalized as a tensor contracted product. There exists $\mathcal{B} \in \mathbb{R}^{d_{1}\times \hdots \times d_{p} \times d_{p+1} \times \hdots \times d_{p+q}} $. $P$ denotes as first $p$ modes contracted product.  
\begin{equation}\label{eq:md1}
 \mathcal{Y}_{m}  = \llangle \mathcal{X}_{m}, \mathcal{B} \rrangle_{P} + \varepsilon^{m}
\end{equation}
So our goal is to estimate $\mathcal{B}$ given the i.i.d samples $\{{\mathcal{X}_{m}, \mathcal{Y}_{m}}\}$. In order to reduce the complexity and enforce the interpretability and structural information, we impose the sparse and low rankness constraints on regression. Particularly, we assume $\mathcal{B}$ can be CP decomposed, $\mathcal{B} = \sum_{r=1}^R w \circ \pmb{\beta}_{1}^{r}\circ \hdots \circ  \pmb{\beta}_{p}^{r} \circ \pmb{\beta}_{p+1}^{r} \hdots  \circ  \pmb{\beta}_{p+q}^{r}$, where $\parallel \pmb{\beta}_i^{r} \parallel = 1$.  Here we call first $p$ vectors are contraction vectors, last $q$ vectors are generation vectors, and $w$ is scalar value indicated the multitude of the coefficient tensor. So the problem (\ref{eq:md1}) can be rewritten into (\ref{eq:md2})

\begin{equation}\label{eq:md2}
 \begin{aligned}
      &\min_{\mathcal{B}} \frac{1}{M} \sum_{m=1}^{M}\parallel \mathcal{Y}_{m} - \llangle \mathcal{B}, \mathcal{X}_{m} \rrangle_{P} \parallel_F^2 + \alpha \parallel \mathcal{B} \parallel_{F}^2 + \lambda \parallel \mathcal{B} \parallel_1 \\
      & s.t.  \;  \mathcal{B} = \sum_{r=1}^R w\circ \pmb{\beta}_{1}^{r}\circ \hdots \circ \pmb{\beta}_{p}^{r} \circ  \pmb{\beta}_{p+1}^{r} \hdots  \circ \pmb{\beta}_{p+q}^{r}
     \end{aligned}
\end{equation}

\begin{lemma}\label{le1}

Given two tensors  $\mathcal{A} \in \mathbb{R}^{d_{1}\times ... \times d_{p}}$ and  $\mathcal{B} \in \mathbb{R}^{d_{1}\times ... \times d_{p} \times d_{p+1} \times ... \times d_{p+q}}$, if $\mathcal{B}$ can be CP decomposed into rank $1$ vectors,  $\mathcal{B} = b \circ \pmb{\beta}_{1}\circ \hdots  \pmb{\beta}_{p} \circ  \pmb{\beta}_{p+1} \hdots \circ  \pmb{\beta}_{p+q} $. Then the contracted product $\llangle \mathcal{A}, \mathcal{B} \rrangle_{P} =\langle \mathcal{A}, \mathcal{B}_{P} \rangle \otimes \mathcal{B}_{Q} $ where $\mathcal{B}_{P} = b\circ \pmb{\beta}_{1}\circ \hdots \pmb{\beta}_{p}$ and $\mathcal{B}_{Q} =\pmb{\beta}_{p+1} \hdots \pmb{\beta}_{p+q} $
\end{lemma}

\begin{proof}
\begin{equation*}
\begin{aligned}
    & \llangle \mathcal{A}, \mathcal{B} \rrangle_{P} [i_{p+1},\hdots, i_{p+q}] \\
    & = \sum_{i_{1},\hdots i_{p}}
    \mathcal{A}[i_{1} \hdots i_{p}] \mathcal{B}[i_{1} \hdots i_{p},i_{p+1}\hdots i_{p+q}] \\
    & = \sum_{i_{1},\hdots i_{p}} \mathcal{A}[i_{1},\hdots i_{p}][b \circ \pmb{\beta}_{1}\circ \hdots \pmb{\beta}_{p} \circ \pmb{\beta}_{p+1} \hdots \circ \pmb{\beta}_{p+q}][[i_{1} \hdots i_{p},i_{p+1}\hdots i_{p+q}]]\\
    & = \sum_{i_{1},\hdots i_{p}} \mathcal{A}[i_{1},\hdots i_{p}][ \mathcal{B}_{P} \otimes \mathcal{B}_{Q}][i_{1} \hdots i_{p},i_{p+1}\hdots i_{p+q}]\\
    & =  (\sum_{p_{1},\hdots p_{d}}\mathcal{A}[i_{1},\hdots i_{p}]\mathcal{B}_{P}[i_{1} \hdots i_{p}]) \otimes \mathcal{B}_{Q}[i_{p+1}\hdots i_{p+q}]\\
    & =   \langle \mathcal{A}, \mathcal{B}_{P} \rangle  \otimes \mathcal{B}_{Q}[i_{p+1}\hdots i_{p+q}]
\end{aligned}
\end{equation*}
\end{proof}

Given the Lemma \ref{le1} , a unit rank contracted product under CP decomposition constraint, can be decomposed into an inner product and a outer product. Specifically, the inner product contracts the tensor predictors $\mathcal{X}$ to a scalar then the outer product maps scalar which can be seen as 1 dimensional hidden space to the $\mathcal{Y}$ tensor space. So for a specific rank $r$, where $1\leq r \leq R$, (\ref{eq:md2}) can be transformed into
\begin{equation}\label{eq:md3}
 \begin{aligned}
      &\argmin_{\mathcal{B}} \frac{1}{M} \sum_{m=1}^{M}\parallel \mathcal{Y}_{m} - \langle \mathcal{X}_{m}, \mathcal{B}_{p}^{r}\rangle \otimes \mathcal{B}_{q}^{r}\ \parallel_F^2 \\
      & +\alpha \parallel \mathcal{B}_{p}^{r} \otimes \mathcal{B}_{q}^{r} \parallel_{F}^2 + \lambda \parallel \mathcal{B}_{p}^{r} \otimes \mathcal{B}_{q}^{r} \parallel_1 \\
      & s.t.  \;  \mathcal{B}_{P}^{r} = w_{p} \circ \pmb{\beta}_{1}^{r}\circ \hdots \circ  \pmb{\beta}_{p}^{r} ;
      \ \mathcal{B}_{Q}^{r} =   w_{q} \circ \pmb{\beta}_{p+1}^{r} \hdots  \circ \ \pmb{\beta}_{p+q}^{r}
     \end{aligned}
\end{equation}

\begin{remark}
From previous analysis, the unit rank tensor on tensor regression (\ref{eq:md2}), can be simply transformed into a tensor to scalar regression and a tensor decomposition problem. We will give the rigorous proof in next section.
\end{remark}
Due to the equivalence between N-mode product and inner product \citep{he2018boosted, kolda2009tensor}, the (\ref{eq:md3}) can be rewritten into following (\ref{eq:4})
\begin{equation}\label{eq:4}
 \begin{aligned}
      &\argmin_{\mathcal{B}_P^{r}, \mathcal{B}_Q^{r}}\frac{1}{M} \sum_{m=1}^{M}\parallel \mathcal{Y}_{m} - (\mathcal{X}_{m} \times_{1} \ w_{p}\pmb{\beta}_{1}^{r} \hdots \times_{p}  \pmb{\beta}_{p}^{r}) \otimes \mathcal{B}_{Q}^{r} \parallel_F^2 \\
      &+ \alpha \parallel  \mathcal{B}_{P}^{r} \otimes \mathcal{B}_{Q}^{r}  \parallel_{F}^2 + \lambda \parallel  \mathcal{B}_{P}^{r} \otimes \mathcal{B}_{Q}^{r} \parallel_1 \\
        & s.t.  \;  \mathcal{B}_{P}^{r} =  \ w_{p} \circ \pmb{\beta}_{1}^{r}\circ \hdots \circ  \pmb{\beta}_{p}^{r} ;
      \ \mathcal{B}_{Q}^{r} = w_q \circ \pmb{\beta}_{p+1}^{r} \hdots  \circ  \pmb{\beta}_{p+q}^{r}
     \end{aligned}
\end{equation}
(\ref{eq:4}) shows a way to turn tensor-wised optimization into a vector-wised optimization. 
\section{Decomposable Sparse Tensor on Tensor Regression}
In this section,  we will discuss the details about  optimization of the object function.
We propose the decomposable sparse tensor on tensor regression denoted as DST2R. Decomposable regression naturally comes from decomposition of the unit rank coefficient tensor $\mathcal{B}$ by $\mathcal{B}_Q \otimes \mathcal{B}_{P}$ which are optimized alternatively. We denote the optimization of $\mathcal{B}_{P}$ as contraction tensor and $\mathcal{B}_{Q}$ as generation tensor.
\begin{theorem}\label{th1}
For a unit rank coefficient tensor $\mathcal{B}$,  the problem of (\ref{eq:md2}) can be solved by a constrained tensor to scalar regression when fixing $\mathcal{B}_{Q}$ and a constrained tensor decomposition by fixing $\mathcal{B}_{P}$
\end{theorem}
\begin{proof}(sketch)
Firstly fix $\mathcal{B}_Q$, we have following objective function  $\argmin_{\mathcal{B}_{P}} \frac{1}{M} \sum_{m=1}^{M}\parallel \mathcal{Y}_{m} - \langle \mathcal{X}_{m} , \mathcal{B}_P^{r}\rangle \otimes \mathcal{B}_{Q}^{r} \parallel_F^2 + \alpha \parallel  \mathcal{B}_{P}^{r}  \parallel_{F}^2 + \lambda \parallel  \mathcal{B}_{p}^{r} \parallel_1 $ such that $\mathcal{B}_{P}^{r} = \pmb{\beta}_{p_{1}}^{r}\circ \hdots \circ  \pmb{\beta}_{p_{n}}^{r}$. Since $\mathcal{B}_{Q}$ is fixed, we can flatten $\mathcal{Y}$ and $\mathcal{B}_Q$ to vector space.  $\argmin_{\mathcal{B}_{P}} \frac{1}{M} \sum_{m=1}^{M}\parallel Vec(\mathcal{Y}_{m}) - \langle \mathcal{X}_{m} , \mathcal{B}_P^{r}\rangle Vec(\mathcal{B}_{Q}^{r}) \parallel_F^2 + \alpha \parallel  \mathcal{B}_{p}^{r}  \parallel_{F}^2 + \lambda \parallel  \mathcal{B}_{p}^{r} \parallel_1 $, given $\langle \mathcal{X}_m, \mathcal{B} \rangle$ is a scalar. This problem is equivalent to  $\argmin_{\mathcal{B}_{P}} \frac{1}{M} \sum_{m=1}^{M} \sum_{k=1}^{d_{u+1}\hdots d_{u+v}} \parallel \frac{Y_{m}^{*}[k]}{B_{Q}^{*}[k]}  - \langle \mathcal{X}_{m} , \mathcal{B}_P^{r}\rangle  \parallel_F^2 $, where $B_{Q}^{*}$ and ${Y_{m}}^{*}$ are vectorized $\mathcal{B}_Q$ and $\mathcal{Y}_{m}$ with removing 0 entries in  $\mathcal{B}_Q$ and corresponding $\mathcal{Y}_{m}$ at the same coordinates. Clearly this is a tensor to scalar regression problem as defined in previous work \citep{he2018boosted}. \\
Next let's fix  $\mathcal{B}_P$ and optimize $\mathcal{B}_Q$. The problem turns into $\argmin_{\mathcal{B}_{Q}} \frac{1}{M} \sum_{m=1}^{M}\parallel \mathcal{Y}_{m} - \langle \mathcal{X}_{m} , \mathcal{B}_P^{r}\rangle \otimes \mathcal{B}_{Q}^{r} \parallel_F^2 + \alpha \parallel  \mathcal{B}_{Q}^{r}  \parallel_{F}^2 + \lambda \parallel  \mathcal{B}_{Q}^{r} \parallel_1 $ such that $\mathcal{B}_{q}^{r} = \pmb{\beta}_{p+1}^{r}\circ \hdots \circ  \pmb{\beta}_{p+q}^{r}$. Given $\langle \mathcal{X}_{m} , \mathcal{B}_P^{r}\rangle$ is a scalar, the problem is equivalent to $\argmin_{\pmb{\beta}_{p+1}^{r} \hdots \pmb{\beta}_{p+q}^{r}} \frac{1}{M} \sum_{m=1}^{M}\parallel \mathcal{Y}_{m} - w_{m, Q} \pmb{\beta}_{p+1}^{r}\circ \hdots \circ  \pmb{\beta}_{p+q}^{r} \parallel_F^2 + \alpha \parallel  \mathcal{B}_{Q}^{r}  \parallel_{F}^2 + \lambda \parallel  \mathcal{B}_{Q}^{r} \parallel_1 $ such that $\mathcal{B}_{q}^{r} = \pmb{\beta}_{p+1}^{r}\circ \hdots \circ  \pmb{\beta}_{p+q}^{r}$. This is a standard tensor decomposition problem \citep{sun2017provable}. 

\end{proof}
Given the theorem (\ref{th1}), DST2R model can be illustrated in figure (\ref{mdl1}).

\subsection{Contraction Part}
$\langle \mathcal{X}, \mathcal{B}_{P} \rangle$ is called contraction part given mapping a tensor to a one dimensional hidden space(scalar). The optimization of the contraction tensor $ \mathcal{B}_{P}^r$ is searching for the optimal $\mathcal{B}_{P}^{r}$ for a  rank $r$, with $\mathcal{B}_{Q}^{r}$ is fixed. Given $\mathcal{B}_{P}^{r}$ is constrained by CP decomposition, the loss function can be defined as (\ref{c1})
\begin{equation}\label{c1}
 \begin{aligned}
      \argmin_{\mathcal{B}_{P}} &\frac{1}{M} \sum_{m=1}^{M}\parallel \mathcal{Y}_{m} - \langle \mathcal{X}_{m} , \mathcal{B}_P^{r}\rangle \otimes \mathcal{B}_{Q}^{r} \parallel_F^2 \\
      &+ \alpha \parallel  \mathcal{B}_{p}^{r}  \parallel_{F}^2 + \lambda \parallel  \mathcal{B}_{p}^{r} \parallel_1 \\
        & s.t.  \;  \mathcal{B}_{p}^{r} =  w_{p}^{r} \circ \pmb{\beta}_{1}^{r}\circ \hdots \circ  \pmb{\beta}_{p}^{r}
     \end{aligned}
\end{equation}

Let $w_{p}^{r} \geq 0$, $\parallel \pmb{\beta}_{k}^{r}\parallel_{1} = 1$, so that $ \pmb{\beta}_{k}^{r}$, $1 \leq k \leq p$ is identifiable up to sign flipping \citep{he2018boosted}. (\ref{c1}) can be reformalized into (\ref{c2})
\begin{equation}\label{c2}
 \begin{aligned}
      \argmin_{\pmb{\beta}_{k}^{r}} &\frac{1}{M} \sum_{m=1}^{M}\parallel \mathcal{Y}_{m} - (\mathcal{X}_{m} \times_{1} \ w_{p}\pmb{\beta}_{1}^{r} \hdots \times_{p}  \pmb{\beta}_{p}^{r}) \otimes \mathcal{B}_{Q}^{r} \parallel_F^2 \\
      &+ \alpha  w_{p}^{r2} \parallel  \pmb{\beta}_{k}^{r}  \parallel_{F}^2 + \lambda w_{p}^{r} \parallel  \pmb{\beta}_{k}^{r} \parallel_1 \\
        & s.t.   w_{p} \geq 0 , \;   \parallel  \pmb{\beta}_{k}^{r} \parallel_{1} = 1, k = 1, ..., p;
     \end{aligned}
\end{equation}

As the theorem \ref{th1} shown, contraction part can be seen as a low rank sparse tensor to scalar regression. Inspired by previous works \citep{he2018boosted, kolda2009tensor, sun2017store}, each mode pairs $(w_p , \pmb{\beta_{k}})$, $1 \leq k \leq p$ is optimized alternatively. Denote $\hat{\pmb{\beta}_{k}} = w_{p}\pmb{\beta}_{k}$. So (\ref{c2}) can be reformulated into (\ref{c3}).
\begin{equation} \label{c3}
    \begin{aligned}
     \argmin_{\hat{\pmb{\beta}_{k}}}
     & \frac{1}{M} \sum_{m=1}^{M}\parallel Vec(\mathcal{Y}_{m}) -\mathbf{Z}^{(-k)T}_{m(1)} \hat{\pmb{\beta}_{k}} \parallel_2^2 \\
     & + \alpha \sigma^{(-k)} \parallel \ \hat{\pmb{\beta}_{k}} \parallel_{2}^2 + \lambda \parallel \hat{\pmb{\beta}_{k}} \parallel_1 \\
    \end{aligned}
\end{equation}
$Vec(\mathcal{Y}_m)$ is the vectorization of $m$th sample $\mathcal{Y}_m$ with dimension $d_{p+1}d_{p+2}\hdots d_{p+q} \times 1$; $\mathbf{Z}^{(-k)}_{m} = \mathcal{X}_{m}\times_{1} \hat{\pmb{\beta}}_{1} \hdots \times_{k-1} \hat{\pmb{\beta}}_{k-1} \times_{k+1} \hat{\pmb{\beta}}_{k+1} \hdots \times_{p} \hat{\pmb{\beta}}_{p}) \otimes \mathcal{B}^{r}_Q$ with dimension $d_{k}\times d_{p+1}\times \hdots \times d_{p+q}$; $\mathbf{Z}_{m(1)}^{(-k)}$ is the matricization of the first mode, hence the dimension is $d_{k} \times d_{p+1} \hdots d_{p+q}$. $\sigma^{-k} = \prod_{l \neq k}^{p+q} \parallel \pmb{\beta}_{l} \parallel_{2}^2  $.
In order to simplify the computation, we introduce the augment variables, $\hat{\mathbf{y}}_{m} =[Vec(\mathcal{Y}_{m}), 0]^{T}$, $\hat{\mathbf{Z}}^{(-k)}_{m} = [\mathbf{Z}^{(-k)}_{m, (1)}, \sqrt{\alpha \sigma^{(-k)}}\mathbf{I}_{k}]$, $\mathbf{\hat{e}}_{m}= \hat{\mathbf{y}}_{m} - \hat{\mathbf{Z}}^{(-k)}_{m}  \hat{\pmb{\beta}}_{k}$, where $I_{k}$ is the unit vector with length $d_{k}$, $1 \leq k \leq p$. Then the object function can be rewritten into (\ref{c4})
\begin{equation} \label{c4}
    \begin{aligned}
     \argmin_{\pmb{\beta}_{k}}
     & \frac{1}{M} \sum_{m=1}^{M}\parallel \hat{\mathbf{y}}_{m} -\hat{\mathbf{Z}}^{(-k)}_{m} \hat{\pmb{\beta}}_{k} \parallel_2^2 + \lambda \parallel  \pmb{\beta}_{k} \parallel_1 \\
    \end{aligned}
\end{equation}

(\ref{c4}) is a standard lasso objective function. Inspired by stagewise search lasso \citep{hastie2007forward} and the extended work on sparse tensor regression \citep{he2018boosted}, we introduce decomposable sparse tensor on tensor regression denoted DST2R. The general idea of a stagewise search is to  gradually increase or decrease the values of coefficients of the model after appropriate initialization. For the linear regression, the forward step / backward search is to find the best predictor in terms of current residual and increase/decrease its coefficient by a small step. Both contraction part and generation part adopt stagewise search . \par
In order to simplify the notation of (\ref{c4}), we define (\ref{c5})
\begin{equation} \label{c5}
    \begin{aligned}
     J(\pmb{\beta}_{k}) = L(\pmb{\beta}_{k}) + \mathcal{R}({\pmb{\beta}_{k}})
    \end{aligned}
\end{equation}
Here $J(\cdot)$ denotes the object function, $L(\cdot)$ is the $\parallel \cdot \parallel_{F}$ term, and $\mathcal{R}(\cdot)$ is the $\ell_1$ norm regularizer. \par
Firstly, let's define $s = \pm \epsilon$ which is the step size controlling the fineness of the searching grid. During the backward stage, the optimal index/coordinate $i_{k}$ to be selected for a specific $ \hat{\pmb{\beta}_{k}}$,  where $1 \leq i_{k} \leq d_{k}, \;  1 \leq k \leq p$. Hence $L( \pmb{\beta}_{k} - sI_{i_{k}})$ where $I_{i_{k}}$ is a $d_{k}$ length vector with $i_{k}$th entry being 1, rest 0. We define the backward stage search which decreases of the value of $\pmb{\beta}_{k}$ $i_{k}$ th entry. 
\begin{equation} \label{c6}
    \begin{aligned}
    L(\hat{\pmb{\beta}}_{k} - s I_{i_{k}})= & \frac{1}{M} \sum_{m=1}^{M} Tr(\hat{\mathbf{e}}_{m}^{T}\hat{\mathbf{e}}_{m}) \\
    & +s^{2} Tr(I_{i_{k}}^{T}\mathbf{\hat{Z}_{m}}^{(-k)} \mathbf{\hat{Z}_{m}}^{(-k)T} I_{k}) \\
   & +2s Tr(I_{i_{k}}^{T}\hat{\mathbf{Z}_{m}}^{(-k)}\mathbf{\hat{e}}_{m})
    \end{aligned}
\end{equation}
Where  $\mathbf{\hat{e}}_{m}= \hat{\mathbf{y}}_{m} - \hat{\mathbf{Z}}^{(-k)}_{(m)} \hat{\pmb{\beta}_{k}}$, $Tr(\cdot)$ is the trace of a matrix, and $s = sign(\hat{\pmb{\beta}_{k}}[i_{k}]) * \epsilon$.

Similarly, we define the forward stage  $L( \hat{\pmb{\beta}}_{k} + sI_{i_{k}})$ which increases the value of of $\pmb{\beta}_{k}$ $i_{k}$ th entry. 
\begin{equation} \label{c7}
    \begin{aligned}
    L(\hat{\pmb{\beta}} + s I_{i_{k}})= &\frac{1}{M} \sum_{m=1}^{M} Tr(\mathbf{\hat{e}}_{m}^{T}\mathbf{\hat{e}}_{m})\\
    & + s^{2} Tr(I_{i_{k}}^{T}\hat{\mathbf{Z}}_{m}^{(-k)T} \hat{\mathbf{Z}}_{m}^{(-k)} I_{i_{k}}) \\
    & -2s Tr( I_{i_{k}}^{T}\hat{\mathbf{Z}}_{m}^{(-k)}\mathbf{\hat{e}}_{m})
    \end{aligned}
\end{equation}

 Since for each iteration, the term $Tr(\mathbf{\hat{e}}_{m}^{T}\mathbf{\hat{e}}_{m})$ is a constant, so we can just drop this term. Furthermore, the calculation of trace can be vectorized into $Diag$ as following term to reduce the computation as (\ref{s1}), (\ref{s2}).
\begin{equation}
\begin{aligned}
(k^{*}, i_{k}^{*}) = 
& \argmin_{i_{k}} \frac{1}{M} \sum_{m=1}^{M}  s^{2} Diag(\hat{\mathbf{Z}}_{m}^{(-k)}\hat{\mathbf{Z}}_{m}^{(-k)T}) \\
& +2s (\hat{\mathbf{Z}}^{(-k)}\mathbf{\hat{e}}_{m})
\end{aligned}
\label{s1}
\end{equation}
\vline
\begin{equation}
\begin{aligned}
(k^{*}, i_{k}^{*}) = 
& \argmin_{i_{k}} \frac{1}{M} \sum_{m=1}^{M}  s^{2} Diag(\hat{\mathbf{Z}}_{m}^{(-k)}\hat{\mathbf{Z}}_{m}^{(-k)T}) \\
& -2s (\hat{\mathbf{Z}}^{(-k)}\mathbf{\hat{e}}_{m})
\end{aligned}
\label{s2}
\end{equation}
The procedure of contraction part is \textbf{Algorithm 2}

\begin{algorithm}
\caption{The Procedure of DST2R}
\begin{algorithmic}
\Procedure{DST2R}{$\epsilon, R, \gamma$}  
    \State Initialize $\mathcal{B}_P^{1}, \mathcal{B}_Q^{1}, \lambda^{1}_{0}$, Res=$\mathcal{Y}$, $r=1, t=0, \sigma_{0}$,  
    \For $\,  r = 1 : R$
    \While{$\lambda^{r}_{t} \geq 0$} 
        \State Calculate error term $\mathbf{\hat{e}}^{(r)}$
        \State Select the optimal $k$ and $i_{k}$ by \textbf{contract}$(\mathcal{B}^{r} ,\mathbf{\hat{e}}^{(r)}, \epsilon, \gamma, \hat{\mathbf{Z}}, t, \lambda^{r}_{t})$, for $k = 1 \hdots p$;
        \State Update $\hat{\mathbf{Z}}^{(-k)}$, $\pmb{\beta}_{k}, \sigma_{t}, w_{p}^{t}$
        \State Calculate error term $\mathbf{\hat{e}}^{(r)}$
        \State Select the optimal $k$ and $i_{k}$ by \textbf{generate}$(\mathcal{B}^{r}, \mathbf{\hat{e}}^{(r)}, \epsilon, \gamma, \hat{\mathbf{Z})}, \lambda^{r}_{t})$  for $k = p+1 \hdots p+q$;
        \State Update $\hat{\mathbf{Z}}^{(-k)}$, $\lambda_{t}^{r}$, $\pmb{\beta}_{k}, \sigma_{t}, w_{q}^{t}$
    \EndWhile 
    \State $\textbf{end while}$
    \State Res = $\mathcal{Y} - \mathcal{Y}_{m} - \langle \mathcal{X}_{m} , \mathcal{B}_P^{r}\rangle \otimes \mathcal{B}_{P}$ , $t = t+1$
    \EndFor 
    \State $\textbf{end for}$
\State $\textbf{end}$
\EndProcedure
\end{algorithmic}
\end{algorithm}

\begin{small}
\begin{table*}\centering
\begin{tabular}{|l|l|l|l|l|l|l|l|}
\hline
Rank                                      & Sparsity             & Measurement & \multicolumn{5}{c|}{Methods}                                                                         \\ \hline
\multicolumn{1}{|c|}{}                    &                      &             & \multicolumn{1}{c|}{Sparse OLS} & STORE        & ENV          & HOLRR        & DST2R                 \\ \hline
\multirow{6}{*}{2}                        & \multirow{3}{*}{0.2} & Error         & 56.58 (2.19)                  & 3.53(0.12) & 6.24(0.15) & 4.30(0.14) & \textbf{1.10(0.10)} \\ \cline{3-8} 
                                          &                      & TPR         & 0.95                           & 1.00            & 1.00            & 1.00            & 1.00                    \\ \cline{3-8} 
                                          &                      & FPR         & 0.01                           & 0.00           & 0.00           & 1.00            & 0                     \\ \cline{2-8} 
                                          & \multirow{3}{*}{0.5} & Error         & 60.78(2.21)                   & 3.94(0.14) & 6.96(0.12) & 4.75(0.14) & \textbf{1.16(0.10)} \\ \cline{3-8} 
                                          &                      & TPR         & {0.99(0)}                  & 1.00              & 1.00             & 1.00              & 1.00                      \\ \cline{3-8} 
                                          &                      & FPR         & 0.13(0.00)                    & 0.00            & 0.00           & 1.00           & 0.00                   \\ \hline
\multicolumn{1}{|c|}{\multirow{6}{*}{10}} & \multirow{3}{*}{0.2} & Error         & 62.34(1.98)                   & 3.34(0.14) & 7.97(0.13) & 5.65(0.10) & \textbf{1.23(0.09)} \\ \cline{3-8} 
\multicolumn{1}{|c|}{}                    &                      & TPR         & 0.910                           & 1.00            & {1.00  }      & 1.00              & 1.00                      \\ \cline{3-8} 
\multicolumn{1}{|c|}{}                    &                      & FPR         & 0.005                           & 0.00              & 0.00            & 1.00              & 0.00                     \\ \cline{2-8} 
\multicolumn{1}{|c|}{}                    & \multirow{3}{*}{0.5} & Error         & 65.89(2.04)                   & 3.45(0.12) & 7.15(0.13) & 5.61(0.10) & \textbf{1.28(0.10)} \\ \cline{3-8} 
\multicolumn{1}{|c|}{}                    &                      & TPR         & 0.930                           & 1.00             & 1.00             & 1.00             & {1.00  }               \\ \cline{3-8} 
\multicolumn{1}{|c|}{}                    &                      & FPR         & 0.09                            & 0.00              & 1.00             & 1.00             & 0.00                     \\ \hline
\end{tabular}
\caption{\label{tab1}The results of 3D tensor to 3D tensors regression. Reported are the average estimation error, TPR, FPR for 30 repetition. Standard errors are shown in the parenthesis.  DST2R out performs all other methods across error, TPR and FPR. }
\end{table*}
\end{small}

\subsection{Generation Part}
Similarly, we define the generation part optimization. Specifically, the part of $ z\otimes \mathcal{B}_Q$  is called generation part,  where $z$ is denoted for the result of $\langle \mathcal{X}, \mathcal{B}_P \rangle$.  The optimization of the generation part, that is searching for the optimal $\mathcal{B}_{Q}^{r}$ for a  rank $r$, with $\mathcal{B}_{P}^{r}$  fixed and $\mathcal{B}_{Q}^{r}$ is CP decomposable. So the loss function can be written as following (\ref{g1})
\begin{equation}\label{g1}
 \begin{aligned}
      \argmin_{\mathcal{B}_{Q}^{r}} &\frac{1}{M} \sum_{m=1}^{M}\parallel \mathcal{Y}_{m} - \langle \mathcal{X}_{m} , \mathcal{B}_P^{r}\rangle \otimes \mathcal{B}_{Q}^{r} \parallel_F^2 \\
      &+ \alpha \parallel  \mathcal{B}_{Q}^{r}  \parallel_{F}^2 + \lambda \parallel  \mathcal{B}_{Q}^{r} \parallel_1 \\
        & s.t.  \;  \mathcal{B}_{Q}^{r} =  w_{q}^{r} \circ \pmb{\beta}_{p+1}^{r}\circ \hdots \circ  \pmb{\beta}_{p+q}^{r}
     \end{aligned}
\end{equation}
The same as contraction part, we reformulate (\ref{g1}) to (\ref{g2}) 
\begin{equation}\label{g2}
 \begin{aligned}
      \argmin_{\pmb{\beta}_{k}^{r}} &\frac{1}{M} \sum_{m=1}^{M}\parallel \mathcal{Y}_{m} - (\langle \mathcal{X}_{m} , \mathcal{B}_P^{r}\rangle ) \circ w_q \circ \pmb{\beta}_{p+1}^{r} \circ \hdots \ \pmb{\beta}_{p+q}^{r}  \parallel_F^2 \\
      &+ \alpha  w_{q}^{r2} \parallel  \pmb{\beta}_{k}^{r}  \parallel_{F}^2 + \lambda w_{q}^{r} \parallel  \pmb{\beta}_{k}^{r} \parallel_1 \\
        & s.t.   w_{q} \geq 0 , \;   \parallel  \pmb{\beta}_{k}^{r} \parallel_{1} = 1, k = p+1, ..., p+q;
     \end{aligned}
\end{equation}

\begin{table*}\centering \label{tab2}
\begin{tabular}{|l|l|l|l|l|l|l|}
\hline
Sparsity             & Measurement & \multicolumn{5}{c|}{Methods}                                                                         \\ \hline
                     &             & \multicolumn{1}{c|}{Sparse OLS} & STORE        & ENV          & HOLRR        & DST2R                  \\ \hline
\multirow{3}{*}{0.2} & Error         & 42.14 (1.34)                  & 2.14(0.11) & 7.63(0.15) & 4.24(0.14) & \textbf{1.33(0.09)} \\ \cline{2-7} 
                     & TPR         & 0.99(0.00)                        & 1.00(0.00)         & 1.00(0.00)         & 1.00(0)         & 1.00(0)                  \\ \cline{2-7} 
                     & FPR         & 0.00                           & 0.00(0.00)         & 0.00(0.00)         & 1.00(0.00)         & 0.00(0.00)                  \\ \hline
\multirow{3}{*}{0.5} & Error         & 44.78(1.26)                   & 3.19(0.11) & 8.27(0.11) & 5.32(0.14) & \textbf{1.24(0.12)} \\ \cline{2-7} 
                     & TPR         & {0.99(0.00)}                  & 1.00(0.00)         & 1.00(0.00)         & 1.00(0.00)         & 1.00(0.00)                  \\ \cline{2-7} 
                     & FPR         & 0.13(0.00)                    & 0.00(0.00)         & 0.00(0.00)         & 1.00(0.00)         & 0.00(0.00)                  \\ \hline
\end{tabular}
\caption{The results of 3D tensor to 2D tensors regression. Reported results are the average estimation error, TPR, FPR for 30 repetition. Standard errors are shown in the parenthesis.  DST2R out performs all other methods across error, TPR and FPR. }
\end{table*}

\begin{figure}[t]\label{step}
\begin{center}
    \includegraphics[scale=0.5]{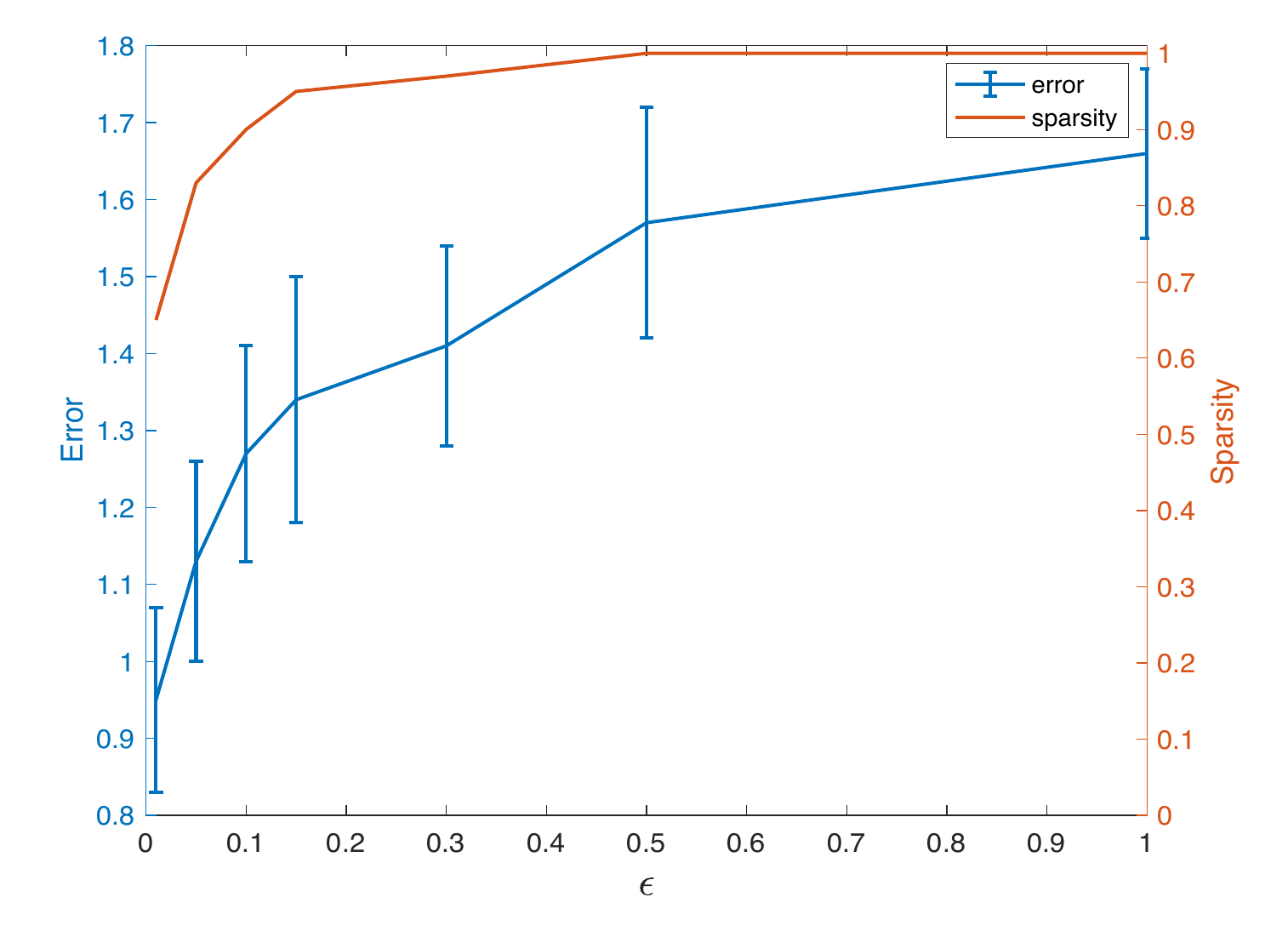}
\caption{This figure shows the $\epsilon$ increases, error and sparsity coverage decreases.}
\end{center}

\end{figure}

The augmented variables are also defined in generation phase. $\hat{\pmb{\beta}_{k}} = w_{q}\pmb{\beta}_{k}$. $\hat{\mathbf{y}}_{m} = [\mathcal{Y}_{m(k)}, 0]$ where $\mathcal{Y}_{m(k)}$ is the  matricization of $\mathcal{Y}_{m}$ in mode k. $\mathbf{Z}^{(-k)}_{m} = \langle \mathcal{X}_{m}, \mathcal{B}_P^{r}\rangle \circ w_q \circ \pmb{\beta}_{p+1}^{r} \circ  \hdots \circ \pmb{\beta}_{k-1}^{r} \circ  \pmb{\beta}_{k+1}^{r} \circ \hdots \circ \pmb{\beta}_{p+q}^{r}$ with dimension $d_{p+1}\times \hdots \times d_{k-1}\times \hdots \times d_{k+1} \times \hdots d_{p+q}$; $\hat{\mathbf{Z}}^{(-k)} = [Vec(\mathbf{Z}^{(-k)}), \sqrt{\alpha \sigma^{-k}}]$,  $\sigma^{-k} = \prod_{l\neq k}^{p+q} \parallel \pmb{\beta}_{k}^{r} \parallel_{2}^{2}$,  $\mathbf{\hat{e}}_{m}= \hat{\mathbf{y}}_{m} - \hat{\pmb{\beta}}_{k}\hat{\mathbf{Z}}^{(-k)}_{m}$. 

The same as contraction part, The backward and forward search is defined as following (\ref{g3}, \ref{g4})
\begin{equation}
\label{g3}
\begin{aligned}
(k^{*}, i_{k}^{*}) = 
& \argmin_{k, i_{k}} \frac{1}{M} \sum_{m=1}^{M}  s^{2} Diag(\hat{\mathbf{Z}}_{m}^{(-k)} \hat{\mathbf{Z}}_{m}^{(-k)T}) \\
& +2s (\hat{\mathbf{Z}}^{(-k)}\mathbf{\hat{e}}_{m}^{T})
\end{aligned}
\end{equation}
\vline
\begin{equation}
\label{g4}
\begin{aligned}
(k^{*}, i_{k}^{*}) = 
& \argmin_{k,i_{k}} \frac{1}{M} \sum_{m=1}^{M}  s^{2} Diag(\hat{\mathbf{Z}}_{m}^{(-k)}\hat{\mathbf{Z}}_{m}^{(-k)T}) \\
& -2s (\hat{\mathbf{Z}}^{(-k)}\mathbf{\hat{e}}_{m}^{T})
\end{aligned}
\end{equation}
 After contraction and  generation phase, $\lambda$ also gets updated. Intuitively, the selection of the index $(k, i_{k})$ is guided by minimizing the regularized $L(\cdot)$ with the current $\lambda_{t}$ and step size. The details of procedure of generation optimization is shown in \textbf{Algorithm 3} 



\begin{algorithm}\label{contract}
\caption{The algorithm \textbf{contract}}
\begin{algorithmic}
\Procedure{\textbf{contract}}{$\mathcal{B} ,\mathbf{\hat{e}}, \epsilon, \gamma, \hat{\mathbf{Z}}, \lambda)$}  
    \State Initialization
    \State \textbf{Backward Search}:
    \State $\begin{aligned}
        (k^{*}, i_{k}^{*}) = 
        & \argmin_{i_{k}} \frac{1}{M} \sum_{m=1}^{M}  s^{2} Diag(\hat{\mathbf{Z}}_{m}^{(-k)}\hat{\mathbf{Z}}_{m}^{(-k)T}) \\
        & +2s (\hat{\mathbf{Z}}^{(-k)}\mathbf{\hat{e}}_{m})\\
      & where \; s = +sign(\pmb{\beta}_{k}^{t}[i_{k}]*\epsilon)
    \end{aligned}$

    \State
    \If {$J(\pmb{\beta}_{k^{*}}^{t} - sI_{i_{k}}) - J(\pmb{\beta}_{k^{*}}^{t}) < -\gamma$}
    \State $\mu_{t+1} = \parallel \beta_{k^{*}}^{t} - sI_{i_{k}} \parallel_{1}$;  $\pmb{\beta}_{k^{*}}^{t+1} = \pmb{\beta}_{k^{*}}^{t} - sI_{i_{k*}}$,
    \Else \; \textbf{Forward Search}:
    \State
    $\begin{aligned}
         (k^{*}, i_{k}^{*}) = 
        & \argmin_{i_{k}} \frac{1}{M} \sum_{m=1}^{M}  s^{2} Diag(\hat{\mathbf{Z}}_{m}^{(-k)}\hat{\mathbf{Z}}_{m}^{(-k)T}) \\
        & -2s (\hat{\mathbf{Z}}_{m}^{(-k)}\mathbf{\hat{e}}_{m})\\
      & where \; s = +sign(\pmb{\beta}_{t, k}[i_{k}]*\epsilon)
    \end{aligned}$
    \State $\mu_{t+1} = \parallel \beta_{k^{*}}^{t} + sI_{i_{k}} \parallel_{1}$;  $\pmb{\beta}_{k^{*}}^{t+1} = \pmb{\beta}_{k^{*}}^{t} + sI_{i_{k}}$,

    \EndIf
    \State \textbf{Return}  $k^{*}, i_{k}^{*}, \mu_{t+1}, \lambda_{t+1}, \mathcal{B}$
\EndProcedure

\end{algorithmic}
\end{algorithm}

\begin{algorithm}
\label{generate}
\caption{The Procedure of \textbf{generate}}
\begin{algorithmic}
\Procedure{\textbf{generate}}{$\mathcal{B} ,e^{t}, \epsilon, \gamma, Z, t, \lambda$}
        \State Initialization
    \State \textbf{Backward Search}:
    \State $\begin{aligned}
        (k^{*}, i_{k}^{*}) = 
        & \argmin_{i_{k}} \frac{1}{M} \sum_{m=1}^{M}  s^{2} Diag(\hat{\mathbf{Z}}_{m}^{(-k)}\hat{\mathbf{Z}}_{m}^{(-k)T}) \\
        & +2s (\hat{\mathbf{Z}}_{m}^{(-k)}\mathbf{\hat{e}}_{m}^{T})\\
       & where \; s = +sign(\pmb{\beta}_{k}^{t}[i_{k}]*\epsilon)
    \end{aligned}$

    \If {$J(\pmb{\beta}_{k^{*}}^{t}- sI_{i_{k^{*}}}) - J(\pmb{\beta}_{k^{*}}^{t}< -\gamma$}
    \State $\mu_{t+1} = \parallel \pmb{\beta}_{k^{*}}^{t} -sI_{i_{k^{*}}} \parallel_{1}$, $\pmb{\beta}_{k^{*}}^{t+1} =\pmb{\beta}_{k^{*}}^{t}[i_{k}] -sI_{i_{k^{*}}}$,

    \Else \; \textbf{Forward Search}:
    \State
    $\begin{aligned}
       (k^{*}, i_{k}^{*}) = 
        & \argmin_{i_{k}} \frac{1}{M} \sum_{m=1}^{M}  s^{2} Diag(\hat{\mathbf{Z}}_{m}^{(-k)}\hat{\mathbf{Z}}_{m}^{(-k)T}) \\
        & -2s (\hat{\mathbf{Z}}_{m}^{(-k)}\mathbf{\hat{e}}_{m}^{T})\\
       & where \; s = +sign(\pmb{\beta}_{k}^{t}[i_{k}]*\epsilon)
    \end{aligned}$
     \EndIf
 \State $\mu_{t+1} = \parallel \beta_{k^{*}}^{t} + sI_{i_{k^{*}}} \parallel_{1}$;  $\pmb{\beta}_{k^{*}}^{t+1} = \pmb{\beta}_{k^{*}}^{t} + sI_{i_{k^{*}}}$,
 \State $\lambda_{t+1} = \min[\lambda_{t}, \frac{L(\mu_{t},  \pmb{\beta}_{k^{*}}^{t}) - L(\mu_{t}, \pmb{\beta}_{k^{*}}^{t}) - \sigma}{\Omega(\mu_{t+1}, \pmb{\beta_{k^{*}}^{t+1}}) - \Omega{(\mu_{t},\pmb{\beta_{k^{*}}^{t}})}}] $
  \State \textbf{Return} $ k^{*}, i_{k}^{*}, \mu_{t+1}, \lambda_{t+1}, \mathcal{B}$   
\EndProcedure
\end{algorithmic}
\end{algorithm}


\section{Theory}
\begin{lemma} (complexity)
the complexity of the algorithm for each iteration, the complexity is $O(M\sum_{k\neq k_p^{*}}^{p}(\prod_{i \neq k, k_p^{*}}^{P+Q}d_{i} + 5\prod_{k=1}^{q}d_{k}) + M\sum_{k\neq k_{q}^{*}}^{Q}(\prod_{i \neq k, k_{q}^{*}}^{P+Q}d_{i} + 5\prod_{k=1}^{Q}d_{k}) + 2Md_{k_p^{*}} + 2Md_{k_q^{*}})  $ where $k_{p}^{*}, k_{q}^{*}$ are the index selected in contraction and generation. The proof is shown in supplementary material.
\end{lemma}
\begin{lemma}(initialization)
The predictors $\mathcal{X}$ are $p+1$ mode tensor with dimension $d_{1} \times \hdots \times d_{p} \times M $ after stacking up M predictors, and $X$ is the matrification of $\mathcal{X}$ with size $M\times d_{1}\hdots d_{p} $. The responses $\mathcal{Y}$ is $q$ mode tensor with dimension $d_{p+1} \times \hdots \times d_{p+q} \times M $. $Y$ is the matrification of $\mathcal{y}$ with size $M\times d_{p+1}\hdots d_{p+q} $.  The initialization value of $\lambda_{0}$ is defined 
\begin{equation*}
    \lambda_{0} = \frac{1}{m} \max\{|X^{T}Y|_{[i, j]}, i = 1, \hdots,  \prod_{s=1}^{p} d_{s},  j = 1 \hdots \prod_{s=p+1}^{p+q} d_{s}\}
\end{equation*}
The $\mathcal{B}_{0}$ is initialized by
\begin{equation*}
    I_{i}, I_{j} = \argmax_{I_{i}, I_{j}} \{|X^{T}Y|_{[i, j]}, i = 1 \hdots \prod_{s=1}^{p} d_{s},  j = 1 \hdots \prod_{s=p+1}^{p+q} d_{s}\}
\end{equation*}
$I_{i}$ and $I_{j}$ are the column and row index of the matrix with the maximum value. $[i_{1}^{*} \hdots i_{p}^{*}], [i^{*}_{p+1}, \hdots i^{*}_{p+q}$ are index in the contraction tensor space and generation tensor space. 
$\mu = \epsilon$, $\pmb{\beta}_{1} = sign(X^{T}Y_{[I_{i}, I_{j}]})I_{i{1}^*}$,  $\pmb{\beta}_{u+1} = sign(X^{T}Y_{[I_{p}, I_{q}]})I_{i^{*}_{p+1}}$, $\pmb{\beta}_{k} = I_{i^{*}_{k}}$ where $k = 2 \hdots p; p+2, \hdots p+q $. The $I_{i^{*}_{k}}$ is the with vector $d_k$ length, and $i^{*}_{k}$th element is $1$,  rest elements $0$. The proof can be found in supplementary material.
\end{lemma}
\begin{theorem}
For $t>0$ such that $\lambda_{t+1} < \lambda_{t}$, the pair$(\mu_t, {\pmb{\beta}}_{p},  {\pmb{\beta}}_{q}) $ will converge to $ ({\mu_t^*(\lambda_{t}), {\pmb{\beta}}^{*}_{p}(\lambda_{t}),  {\pmb{\beta}}^{*}_{q}}(\lambda_{t}))$ when $\gamma \rightarrow 0, \epsilon \rightarrow 0$, where $ ({\sigma_t^*(\lambda_{t}),{\pmb{\beta}}^{*}_{p}(\lambda_{t}),  {\pmb{\beta}}^{*}_{q}}(\lambda_{t}))$ denotes the coordinate-wise minimum with subject to $\lambda_{t}$ of problem (\ref{eq:4}). 
\end{theorem}

\section{Experiments}
In order to investigate the performance on accuracy as well as variables selection of our method.  In this section, a series experiments have been done with different settings. We successfully show our methods outperform other related methods. First we define the measurement of the estimation accuracy. For the true weighting tensor $\mathcal{B}$ and learnt $\hat{\mathcal{B}}$, the estimation error is $\parallel\hat{\mathcal{B}} - \mathcal{B}\parallel_{F}$. We also follow the measurement protocol from \citep{sun2017store} that the true positive rate and false positive rate for each mode of weighting tensor are reported. To be specific, for mode j of $\mathcal{B}$, $\pmb{\beta}_{i, j}^{r}$ is the $i$th element of the $\pmb{\beta}_j$ in rank $r$, the true positive rate $TPR^{j}$ and false positive rate $FPR^{j}$ are defined as following

\begin{equation*}
    \begin{aligned}
   & TPR^{j} = \frac{\sum_{r=1}^{R} \sum_{i} 1(\beta^{r}_{j,i} \neq 0, \hat{\beta^{r}_{j,i}} \neq 0)}{R \sum_i 1(\beta_{j,i}^r)}, \\
   & FPR^{j} = \frac{\sum_{r=1}^{R} \sum_{i} 1(\beta^{r}_{j,i}=0, \hat{\beta^{r}_{j,i}} \neq 0)}{R \sum_i 1(\beta_{j,i}^r=0)}, \\
   & TPR = \sum_{j=1}^{p+q} TPR^{j}, FPR = \sum_{j=1}^{p+q} FPR^{j}
    \end{aligned}
\end{equation*}

TPR and FPR illustrate how well the predictors are selected. The TPR is the bigger the better while FPR is the smaller the better. Since DST2R is the one of the first sparse tensor on tensor regression, we compare our methods with other tensor response regression methods by vectorizing the input $\mathcal{X}$ to a vector $X$. The comparison are among sparse tensor response regression STORE \citep{sun2017store} with a vector input, the envelop based tensor response regression from \citep{li2017parsimonious}, and higher order low rank regression (HOLRR) from \citep{rabusseau2016low}, as well as the sparse ordinary least squares method (Sparse OLS) \citep{peng2010regularized} by vectorizing the $\mathcal{X}$ and $\mathcal{Y}$.

\subsection{3D Tensor Predictor to 3D Tensor Responses}
In order to investigate DST2R performance, we first simulate 3D predictor $\mathcal{X}$ and 3D $\mathcal{Y}$ response. We follow \citep{sun2017store} simulation procedures to have a fair comparison. $x_{i,j,k}^{m}$ is the 3D $m$th input $\mathcal{X}_m$ $[i,j,k]$ entry, generated by taking values 0 or 1 with an equal probability 0.5. Then we manually set up the coefficient tensor $\mathcal{B} = \sum_{r=1}^R \pmb{\beta}_{1}^{r}\circ \pmb{\beta}_{2}^{r} \circ \pmb{\beta}_{3}^{r}\circ \pmb{\beta}_{4}^{r}  \circ \pmb{\beta}_{5}^{r} \circ \pmb{\beta}_{6}^r $, $\mathcal{B}_{P} = \sum_{r}^{R} \pmb{\beta}_{1}^{r}\circ \pmb{\beta}_{2}^{r} \circ \pmb{\beta}_{3}^{r}$ and $\mathcal{B}_{Q} = \sum_{r}^{R} \pmb{\beta}_{4}^{r}\circ \pmb{\beta}_{5}^{r} \circ \pmb{\beta}_{6}^{r}$ where $d_{1}=8, d_{2}=8, d_{3}=8$; $d_{4}=4, d_{5}=4, d_{6}=4$, and rank R is in  [2, 10]. 
The $\pmb{\beta}_{k}^{r}$ is generated from $\mathcal{N}(0, I)$. We set the sparsity level in $[0.2, 0.5]$ which indicates the percentage of zero entries in the coefficient tensor. For each $\pmb{\beta}_{k}^{r}$, the number of zero elements is $c_{k}*s$ where $c_{k}$ is the cardinality of vector $\pmb{\beta}_{k}^{r}$. The zero entries are randomly sampled. For each rank r with each sparsity level,  we simulate 1000 ${\mathcal{X}^{m}, \mathcal{Y}^{m}}$ samples. 5 folds cross validation is to find out best parameters of DST2R. And the experiments are repeated 30 times. The results are reported in table \ref{tab1}. From table \ref{tab1},  we can see our method achieve the best result in terms of errors, TPR and FPR compared with other methods. The main reason is that DST2R takes the structural information into account during the optimization.
\subsection{3D Tensor Predictors to 2D Tensor Responses}
DST2R not only can also solve inputs and outputs in different modes number. Here we simulate the $1000$ samples ${\mathcal{X}, \mathcal{Y}}$ where $\mathcal{X}$ are 3D tensors while $\mathcal{Y}$ are 2D tensors. Similar as previous section, $x_{i,j,k}^{m}$  is generated by taking values 0 or 1 with an equal probability which is the entry at $[i,j,k]$ of 3D $m$th input $\mathcal{X}_m$. Then we manually set up the coefficient tensor $\mathcal{B} = \sum_{r=1}^R \pmb{\beta}_{1}^{r}\circ \pmb{\beta}_{2}^{r} \circ \pmb{\beta}_{3}^{r}\circ \pmb{\beta}_{4}^{r}  \circ \pmb{\beta}_{5}^{r}$, $\mathcal{B}_{P} = \sum_{r}^{R} \pmb{\beta}_{1}^{r}\circ \pmb{\beta}_{2}^{r} \circ \pmb{\beta}_{3}^{r}$ and $\mathcal{B}_{Q} = \sum_{r}^{R} \pmb{\beta}_{4}^{r}\circ \pmb{\beta}_{5}^{r}$ where $d_{1}=8, d_{2}=8, d_{3}=8$; 
$d_{3}=4, d_{4}=4$, and rank R is 5. The sparsity level is in $[0.2, 0.5]$. The results are reported in table \ref{tab2}. Similarly as previous scenario, 3D tensor to 2D tensors regression can achieve the best performance compared with other methods.
\subsection{2D Tensor Predictor to 2D Tensor Responses}
In order to examine the how the step size $\epsilon$ influences the errors and sparsity coverage. The sparsity coverage in means the true positive rate for zeros entries. The results are plotted in Figure 2. The results shows that with finer grid of searching step, the sparsity coverage decreases.  We simulate 2D predictors $\mathcal{X}$ with size $16*16$ and 2D responses $\mathcal{Y}$ with size $4*4$ with sample size 1000. We set the $\epsilon$ in  $[0.01, 0.05, 0.1, 0.15, 0.5, 1]$ and sparsity level $s = 0.5$. From the Figure 2, we can tell that with $\epsilon$ increasing, the averaged error increases. Meanwhile, the sparsity percentage shows with increasing $\epsilon$, the sparsity of $\mathcal{B}$ can be learnt increase. The results shows with finer grid of searching step,  the sparsity coverage decreases. 


\section{Conclusion}
In this paper, we decompose the tensor on tensor regression for a unit rank into a tensor to scalar regression(contraction part) and a tensor decomposition (generation part). Hence the optimization can be formulated in two sub-problems. Furthermore, inspired by previous work, we introduce the stagewise search based algorithm DST2R to solve sparse tensor on tensor regression. The experiments result demonstrates DST2R benefits from adopting the structural information from predictors and responses. We hope our work will be useful to those looking to deploy tensor on tensor regression models. 

\renewcommand\refname{Bibliography}
\bibliographystyle{abbrvnat} 
\bibliography{refer}

\section{Appendix}
\section{Proof of Lemma}
\subsection{Proof of Lemma 6.1}
\textit{In \textbf{Algorithm 1} for each iteration there are contraction and generation parts. In contraction part, two main terms $\hat{\mathbf{Z}}_{m}^{(-k)}\hat{\mathbf{Z}}_{m}^{(-k)T}$, $\hat{\mathbf{Z}}_{m}^{(-k)}\mathbf{\hat{e}}_{m}$, need to be computed. Then $\mathbf{Z}_{m}^{(-k)}$ needs to be updated by following
\begin{equation*}
    \begin{aligned}
        \mathbf{Z}^{(-k)}_{m, t+1} = \frac{1}{\sigma_{t+1}}(\sigma_{t} \mathbf{Z}^{(-k)}_{m,t} + \mathbf{Z}^{(-k, -k^{*})}_{m,t} \times_{k*}sI_{i_{k*}} )
    \end{aligned}
\end{equation*}
where$(-k, -k^{*})$ denotes every mode except $k$ and $k^{*}$.
From the updating rule, the complexity of updating rule of  $\mathbf{Z}^{(-k)}$ is $O(M\sum_{k\neq k_p^{*}}^{P}(\prod_{i \neq k, k_p^{*}}^{P+Q}d_{i} + 3\prod_{k=1}^{q}d_{k})$. Updating $e$ has complexity $O(M(\prod_{k=1}^{Q}d_{k})d_{k^{*}})$, updating $\mathbf{Z}^{(-k)}\mathbf{\hat{e}}_{m}$ has complexity $O(M\prod_{k=1}^{Q}d_{k})$, updating $\hat{\mathbf{Z}}_{m}^{(-k)}\hat{\mathbf{Z}}_{m}^{(-k)T}$ has complexity.
Similarly, the generation part, $\hat{\mathbf{Z}}_{m}^{(-k)} \hat{\mathbf{Z}}_{m}^{(-k)T}$, 
$\hat{\mathbf{Z}}^{(-k)}_{m}\mathbf{\hat{e}}_{m}^{T}$, and updating $\mathbf{Z}^{(-k)}$
\begin{equation*}
    \begin{aligned}
        \mathbf{Z}^{(-k)}_{m, t+1} = \frac{1}{\sigma_{t+1}}(\sigma_{t} \mathbf{Z}^{(-k)}_{m,t} + \mathbf{Z}^{(-k, -k^{*})}_{m,t} \circ sI_{i_{k*}} )
    \end{aligned}
\end{equation*}
Even the updating rule are different, the complexity term are the same. Totally we have $O(M\sum_{k\neq k_p^{*}}^{p}(\prod_{i \neq k, k_p^{*}}^{P+Q}d_{i} + 5\prod_{k=1}^{q}d_{k}) + M\sum_{k\neq k_{q}^{*}}^{Q}(\prod_{i \neq k, k_{q}^{*}}^{P+Q}d_{i} + 5\prod_{k=1}^{Q}d_{k}) + 2Md_{k_p^{*}} + 2Md_{k_q^{*}})  $ where $k_{p}^{*}, k_{q}^{*}$ are the optimal mode index in contraction and generation part.
}

\subsection{Proof of Lemma 6.2}
\textit{Let's stack $\mathcal{X}_m$ and $\mathcal{Y}_m$ to tensor $\mathcal{X}$ and $\mathcal{Y}$. Then we perform matricization to come up with $\mathbf{X}$ and $\mathbf{Y}$, where $\mathbf{X} \in \mathbb{R}^{m\times d_{1} \hdots d_{p}}$, $\mathbf{Y} \in \mathbb{R}^{m\times d_{p+1} \hdots d_{p+q}}$ So the regression problem can be rewritten into $\parallel \mathbf{Y} - \mathbf{X}\mathbf{B}\parallel_F^{2}$ where $\mathbf{B} \in \mathbb{R}^{d_1 \hdots d_p \times d_{p+1} \hdots d_{p+q}}$. According to KKT condition, we have the initialization of $\lambda_{0} = max|\mathbf{X}^{T}\mathbf{Y}|_{(i, j)}$ where $(i,j) \in d_1 \hdots d_p \times d_{p+1} \hdots d_{p+q}$ and the initial non-zero solutions as well. }

\section{Proof of Theorem}
In order to prove the \textit{theorem 6.3}, first we need to introduce some lemmas and their proofs. 
\subsection{Lemma 1}
Denote $\epsilon = |\epsilon_P| + |\epsilon_Q|$ to simplify the notation, where $\epsilon_P$ is the step size of contraction part, and $\epsilon_Q$ is the step size of the generation part. In the algorithm $\epsilon_P = \epsilon_Q$. \par
\textbf{Lemma 1} If there exist $i_{1}$ and $i_{p+1}$, where such that $J(sI_{i_1}, \hdots I_{i_p}, sI_{i_{p+1}}, \hdots,I_{i_{p+q}}; \lambda) \leq J({0}; \lambda)$
\textit{Proof: From the assumption, we have $L(sI_{i_1}, \hdots I_{i_p}, sI_{i_{p+1}}, \hdots,I_{i_{p+q}}) + \lambda \mathcal{R}(sI_{i_1}, \hdots I_{i_p}, sI_{i_{p+1}}, \hdots,I_{i_{p+q}})  \leq L({0})$. 
So we have following 
\begin{equation*}
    \begin{aligned}
    \lambda & \leq \frac{1}{\epsilon}(J(\{0\})) - J(\{sI_{i_1}, \hdots I_{i_p}, sI_{i_{p+1}}, \hdots,I_{i_{p+q}}\}) \\
    &\leq \frac{1}{\epsilon}(J\{0\}) - \min_{i_1,...i_{p}} J(\{sI_{i_1}, \hdots I_{i_p}, sI_{i_{p+1}}, \hdots,I_{i_{p+q}}\}) \\
   & \leq \frac{1}{\epsilon}(J\{0\}) - \min_{i_1,...i_{p}, ...i_{p+q}} J(\{sI_{i_1}, \hdots I_{i_p}, sI_{i_{p+1}}, \hdots,I_{i_{p+q}}\})\\
   & = \lambda_{0}
    \end{aligned}
\end{equation*}
}
\subsection{Lemma 2} 
\textbf{Lemma 2} For any $\lambda_{t+1} = \lambda_{t}$ we have $J(\lambda_{t+1}, {\pmb{\beta}^{t+1}_{k_p}}, {\pmb{\beta}^{t+1}_{k_q}}) \leq J(\lambda_{t}, {\pmb{\beta}^{t}_{k_p}}, {\pmb{\beta}^{t}_{k_q}}) - \gamma$ \\
\textit{Proof. If backward step is performed in generation stages, we can easily have $J(\lambda_{t+1}, {\pmb{\beta}^{t+1}_{k_p}}, {\pmb{\beta}^{t+1}_{k_q}}) < J(\lambda_{t}, {\pmb{\beta}^{t}_{k_p}}, {\pmb{\beta}^{t}_{k_q}})-\gamma$ and $\lambda_{t+1} = \lambda_{t}$. Let's consider forward stage. If the claim is not true, then we can have
$L(\sigma_{t}, {\pmb{\beta}^{t}_{k_p}}, {\pmb{\beta}^{t}_{k_q}}) - L(\sigma_{t+1}, {\pmb{\beta}^{t+1}_{k_p}}, {\pmb{\beta}^{t+1}_{k_q}}) < \lambda_{t}\mathcal{R}(\sigma_{t+1}, {\pmb{\beta}^{t+1}_{k_p}}, {\pmb{\beta}^{t+1}_{k_q}}) - \lambda_{t}\mathcal{R}(\sigma_{t}, {\pmb{\beta}^{t}_{k_p}}, {\pmb{\beta}^{t}_{k_q}}) + \gamma = \lambda_{t}\epsilon + \gamma$. So it means $\lambda_{t+1} = \lambda_{t} > \frac{1}{\epsilon}(L(\sigma_{t}, {\pmb{\beta}_{t}^{k_p}}, {\pmb{\beta}^{t}_{k_q}}) - L(\sigma_{t+1}, {\pmb{\beta}^{t+1}_{k_p}}, {\pmb{\beta}^{t+1}_{k_q}}) - \gamma)$ so it get contradicts $\lambda_{t+1} = \min(\lambda_{t}, \frac{1}{\epsilon}(L(\sigma_{t}, {\pmb{\beta}^{t}_{k_p}}, {\pmb{\beta}^{t}_{k_q}}) - L(\sigma_{t+1}, {\pmb{\beta}^{t+1}_{k_p}}, {\pmb{\beta}^{t+1}_{k_q}}) - \gamma))$}

\subsection{Lemma 3}
\textbf{Lemma 3} For any $\lambda_{t+1} < \lambda_{t}$ then $J(\pmb{\beta}_{k_p}^{t} + s_{i_{k_p}}I_{k_{p}}, \pmb{\beta}_{k_q}^{t} + s_{i_{k_q}}I_{k_{q}}, \lambda_{t}) > J(\pmb{\beta}_{k_p}^{t} , \pmb{\beta}_{k_q}^{t} , \lambda_{t}) - \gamma$\\
\textit{Proof. First, if $\lambda_{t+1} < \lambda_{t}$, we have $\mathcal{R}(\sigma_{t+1}, \{\pmb{\beta}_{k_q}^{t+1}, \pmb{\beta}_{k_q}^{t+1}\}) = \mathcal{R}(\sigma_{t}, \{\pmb{\beta}_{k_q}^{t}, \pmb{\beta}_{k_q}^{t}\}) + \epsilon$. From $\lambda_{t+1} = \min(\lambda_{t}, \frac{1}{\epsilon}(L(\sigma_{t}, {\pmb{\beta}^{t}_{k_p}}, {\pmb{\beta}^{t}_{k_q}}) - L(\sigma_{t+1}, {\pmb{\beta}^{t+1}_{k_p}}, {\pmb{\beta}^{t+1}_{k_q}}) - \gamma)))$ and $\lambda_{t+1} < \lambda_{t}$, so we have
$L(\sigma_{t}, {\pmb{\beta}^{t}_{k_p}}, {\pmb{\beta}^{t}_{k_q}}) - L(\sigma_{t+1}, {\pmb{\beta}^{t+1}_{k_p}}, {\pmb{\beta}^{t+1}_{k_q}}) - \gamma = \lambda_{t+1}\epsilon = \lambda_{t+1}(\mathcal{R}(\sigma_{t+1}, \{\pmb{\beta}_{k_q}^{t+1}, \pmb{\beta}_{k_q}^{t+1}\}) - \mathcal{R}(\sigma_{t}, \{\pmb{\beta}_{k_q}^{t}, \pmb{\beta}_{k_q}^{t}\}) )$ Then we have 
\begin{equation*}
    \begin{aligned}
    J(\lambda_{t}, \pmb{\beta}_p^{t}, \pmb{\beta}_q^{t}) - \gamma &= J(\lambda_{t+1}, \pmb{\beta}_p^{t}, \pmb{\beta}_q^{t}) - \gamma + (\lambda_t - \lambda_{t+1})\mathcal{R}(\sigma_{t}, \{\pmb{\beta}_{p}^{t+1}, \pmb{\beta}_{q}^{t+1}\})\\
    & = J(\lambda_{t+1}, \pmb{\beta}_p^{t+1}, \pmb{\beta}_q^{t+1}) + (\lambda_t - \lambda_{t+1})\mathcal{R}(\sigma_{t}, \{\pmb{\beta}_{p}^{t}, \pmb{\beta}_{q}^{t}\})\\
    & =  J(\lambda_{t}, \pmb{\beta}_p^{t+1}, \pmb{\beta}_q^{t+1}) + (\lambda_t - \lambda_{t+1})\mathcal{R}(\sigma_{t+1}, \{\pmb{\beta}_{p}^{t+1}, \pmb{\beta}_{q}^{t+1}\}) - \mathcal{R}(\sigma_{t}, \{\pmb{\beta}_{p}^{t}, \pmb{\beta}_{q}^{t}\}) \\
    & =  J(\lambda_{t}, \pmb{\beta}_p^{t+1}, \pmb{\beta}_q^{t+1}) + (\lambda_t - \lambda_{t+1})\epsilon < J(\lambda_{t}, \pmb{\beta}_p^{t+1}, \pmb{\beta}_q^{t+1}) \\
    &= \min\{J(\lambda_{t}, \pmb{\beta}_{k_p}^{t} + s_{i_{k_p}}I_{k_{p}}, \pmb{\beta}_{k_q}^{t} + s_{i_{k_q}}I_{k_{q}})\}
    \end{aligned}
\end{equation*}
}
\subsection{Proof of Theorem 6.3}
\textit{Proof. Given Lemma 2 we have $J(\lambda_{t+1}, {\pmb{\beta}_{t+1}^{k_p}}, {\pmb{\beta}_{t+1}^{k_q}}) \leq J(\lambda_{t}, {\pmb{\beta}_{t}^{k_p}}, {\pmb{\beta}_{t}^{k_q}}) - \gamma$ when $\lambda_{t
+1} = \lambda_{t}$. So with $\tau$ times iteration, we have following $J(\lambda_{t}, {\pmb{\beta}_{t}^{k_p}}, {\pmb{\beta}_{t}^{k_q}}) \leq J(\lambda_{t-1}, {\pmb{\beta}_{t-1}^{k_p}}, {\pmb{\beta}_{t-1}^{k_q}}) - \gamma \hdots \leq  J(\lambda_{t-\tau}, {\pmb{\beta}_{t-\tau}^{k_p}}, {\pmb{\beta}_{t-\tau}^{k_q}}) - \tau\gamma $. Then we have
\begin{equation*}
    J(\lambda_{t}, {\pmb{\beta}_{t}^{k_p}}, {\pmb{\beta}_{t}^{k_q}}) \leq J(\lambda_{t-1}, {\pmb{\beta}_{t-1}^{k_p}}, {\pmb{\beta}_{t-1}^{k_q}}) \hdots \leq  J(\lambda_{t-\tau}, {\pmb{\beta}_{t-\tau}^{k_p}}, {\pmb{\beta}_{t-\tau}^{k_q}})
\end{equation*}
Given \textbf{Lemma 3} we have $\lambda_{t+1} < \lambda_{t}$ if the forwardstage search performed on generation part. $J(\pmb{\beta}_{k_p}^{t} + s_{i_{k_p}}I_{k_{p}}, \pmb{\beta}_{k_q}^{t} + s_{i_{k_q}}I_{k_{q}}, \lambda_{t}) > J(\pmb{\beta}_{k_p}^{t} , \pmb{\beta}_{k_q}^{t} , \lambda_{t}) - \gamma$ This means after searching all backward stage, with $\lambda_{t}$ the loss function can not be reducefd at more. So when $\epsilon and \gamma \rightarrow 0$ when $\lambda_{t} \rightarrow \lambda{t+1}$, and the object function is convex with $(\sigma, \pmb{\beta}_{p}^t,  \pmb{\beta}_{q}^t)$, so the algorithm reach the coordinate-wise minimum
}

\end{document}